\documentclass[conference]{IEEEtran}

\usepackage{amsmath,amsfonts}
\usepackage[hidelinks]{hyperref}
\usepackage{algorithm}
\usepackage{algpseudocode}
\usepackage{amsthm}
\usepackage{array}
\usepackage{makecell}
\usepackage{mathtools}
\usepackage{booktabs}
\usepackage{cite}
\usepackage[caption=false,font=normalsize,labelfont=sf,textfont=sf]{subfig}
\usepackage{textcomp}
\usepackage{url}
\usepackage{verbatim}
\usepackage{graphicx}
\usepackage[pdftex,dvipsnames,table]{xcolor}

\newcommand{\E}{\mathbb{E}}

\newtheorem{definition}{Definition}
\newtheorem{assumption}{Assumption}
\newtheorem{lemma}{Lemma}
\newtheorem{theorem}{Theorem}

\begin{document}

\title{Regularity-Aware Stochastic MGDA with Adaptive Per-Iterate Conflict-Avoidant Update}

\author{
\IEEEauthorblockN{Chentong Huang and Lisha Chen}
\IEEEauthorblockA{
Department of Electrical and Computer Engineering\\
University of Rochester, Rochester, NY 14627, USA\\
chuang80@ur.rochester.edu; lisha.chen@rochester.edu
}
}

\maketitle

\begin{abstract}
Multi-objective learning (MOL) aims to optimize multiple objectives simultaneously.
The multi-gradient descent algorithm (MGDA) constructs a common descent
or conflict-avoidant (CA) direction by solving a gradient-aggregation
subproblem. In stochastic settings, the vanilla stochastic MGDA method, SMG~\cite{liu2024stochastic}, 
replaces full-batch gradients with mini-batch
estimates, causing bias in the resulting CA direction. Existing analyses control this error using continuity of the CA-direction mapping~\cite{chen2024three}, but do not distinguish how its sensitivity
varies with the geometry of the MGDA subproblem. 
Motivated by this limitation, we show that the MGDA direction map is Lipschitz continuous under
reduced nondegeneracy, whereas its general \(1/2\)-H\"older exponent on
bounded sets is sharp.
This geometry motivates the Multi-objective Regularity-aware (MoRe) method, which retains adaptive MGDA aggregation in
the regular regime and switches to a fixed-weight linear scalarization (LS) near degeneracy.
The fallback branch skips the MGDA subproblem solve, which can
reduce aggregation cost on fallback iterations while avoiding a potentially noise-sensitive update. Theoretically, we establish 
a per-iterate CA guarantee, and a \(\widetilde{\mathcal O}(T^{-1})\) convergence rate,
improving upon the previously established \(\widetilde{\mathcal O}(T^{-1/4})\) SMG guarantee~\cite{chen2024three}, where
\(\widetilde{\mathcal O}\) hides logarithmic factors.
Experiments demonstrate favorable multi-task performance and convergence behavior consistent
with the theoretical rate.

\end{abstract}


\section{Introduction}
Multi-objective learning (MOL) has emerged as an important paradigm in modern machine learning. Representative applications include multi-task learning~\cite{sener2018multi}, meta-learning~\cite{ye2021multi}, and learning under multiple constraints~\cite{zafar2017fairness}. 

In this work, we consider the empirical MOL problem defined on a training sample sequence \(S=(z_1,\dots,z_n)\in\mathcal Z^n\). {Let \([M]\coloneqq\{1,\ldots,M\}\), where \(M\ge2\).} For each objective \(m\in[M]\), let \(f_{z,m}(x)\) denote the \(m\)-th objective function evaluated at model \(x\in\mathbb{R}^p\) on a data point \(z\), and define its empirical counterpart by \(f_{S,m}(x)\coloneqq n^{-1}\sum_{i=1}^n f_{z_i,m}(x)\). The empirical MOL problem is formulated as the optimization of the vector-valued objective:
\begin{equation}
\min_{x\in\mathbb{R}^p} F_S(x)
\coloneqq
\bigl(f_{S,1}(x),\dots,f_{S,M}(x)\bigr)^\top.
\end{equation}

Unlike single-objective learning, MOL seeks a Pareto optimal or Pareto stationary solution rather than the minimizer of a scalar loss. One of the central challenges is that the gradients of different objectives may conflict: a descent direction for one objective may increase another. 
A common
approach is to construct a conflict-avoidant (CA) direction by dynamically
aggregating these gradients~\cite{mukai1980algorithms,liu2021conflict,chen2024three}. Among
the methods, a representative approach is MGDA~\cite{mukai1980algorithms,fliege2000steepest,desideri2012multiple}, which improves all objectives simultaneously and
underlies many MOL variants
~\cite{liu2021conflict,fernando2022mitigating,chen2024three}. Because
full-batch gradients are costly, stochastic variants of MGDA, such as SMG~\cite{liu2024stochastic}, instead use mini-batch gradient estimates.

However, extending MGDA analysis to the stochastic setting is nontrivial because
the MGDA update direction depends nonlinearly on the gradient matrix; consequently,
sampling noise can make the mini-batch direction a biased or unstable
approximation of its deterministic counterpart. 
This sensitivity can be characterized by {$c_{\rm H}$-H\"older continuity} of the direction mapping: for gradient matrices \(Q_1,Q_2\in\mathbb{R}^{p\times M}\),
let $d_{Q_1}$ and $d_{Q_2}$ denote
the corresponding CA directions. Then,
\begin{equation}
\|d_{Q_1}-d_{Q_2}\|
\le
\ell_{\mathrm H}\|Q_1-Q_2\|^{c_{\rm H}},
\end{equation}
where \(\ell_{\mathrm H}>0\) is the H\"older constant,
\(c_{\rm H}\in(0,1]\), and \(c_{\rm H}=1\) corresponds to Lipschitz continuity. Despite its direct relevance to the stability of stochastic MGDA methods, exploration of this question remains limited. Prior work established
sharp \(1/2\)-H\"older continuity with respect to the model variable~\cite{svaiter2018holder} and \(1/2\)-H\"older continuity with respect to
the stochastic gradient matrix~\cite{chen2024three}. These results provide
general sensitivity bounds but leave open whether the MGDA subproblem
geometry admits stronger regularity in certain regimes and can thereby guide
stochastic gradient aggregation. 

\subsection{Related work}
\noindent\textbf{Deterministic gradient-based MOL.} A major line of work in MOL seeks a common update direction for all objectives using gradient information. A foundational method in this direction is MGDA~\cite{mukai1980algorithms,fliege2000steepest,desideri2012multiple}, which dynamically combines gradients to obtain a common descent direction. 
Beyond MGDA, several gradient-aggregation methods, including PCGrad~\cite{yu2020gradient} and
CAGrad~\cite{liu2021conflict}, have been proposed to mitigate conflicts among objectives. Because computing full-batch gradients is costly,
stochastic extensions are more suitable for large-scale learning.\\ 
\textbf{Stochastic MOL.} Practical MOL algorithms usually require updates based on (mini-batch) stochastic gradients, which has motivated stochastic MGDA methods, including
SMG~\cite{liu2024stochastic}, and subsequent provably convergent approaches
based on sampling, variance reduction, and preference guidance~\cite{zhou2022_SMOO,fernando2022mitigating,chen2024three,xiao2023direction,
chen2024ferero,joint_gradient_balancing,fernando2024variance}.
Because the MGDA update
direction depends nonlinearly on the gradient matrix, stochastic noise
induces direction error that can limit convergence rates~\cite{chen2024three}; see Table~\ref{tab:comparison}.
\\
\textbf{Continuity of the CA direction map.} The continuity of
the CA direction mapping is important 
because it quantifies how perturbations of the
stochastic gradient matrix affect the aggregated direction. 
Prior work established \(1/2\)-H\"older continuity with
respect to the model variable~\cite{svaiter2018holder} and the stochastic
gradient matrix~\cite{chen2024three}.
However, these results do not
distinguish how this sensitivity behaves across different geometric regimes
of the MGDA subproblem or show how this geometry can guide aggregation. In this work, we address this gap through a regularity-aware rule
that reduces aggregation cost while improving convergence and CA
guarantees.

\subsection{Our contributions}
\noindent Our contributions in this paper are summarized as follows.\\
\textbf{Continuity analysis of the CA direction.}
We provide a fine-grained continuity analysis of the CA direction mapping. 
Under a nondegeneracy condition in Definition~\ref{def:pl_main}, Lemma~\ref{lem:lip_dq} proves Lipschitz continuity of the CA direction mapping, while Lemma~\ref{lem:holder_sharp} shows that without the nondegeneracy condition, the \(1/2\)-H\"older continuity on bounded sets is sharp.\\
\textbf{Regularity-aware stochastic MGDA.}
Motivated by this continuity analysis, we propose Multi-objective Regularity-aware (MoRe) method, which adapts the
aggregation rule to subproblem regularity and selectively skips MGDA solves,
reducing aggregation cost. With \(|Z_t|=\Theta(t+1)\) and
\(\alpha_t=\Theta(1)\), we establish a
\(\widetilde{\mathcal O}(T^{-1})\) convergence rate, improving upon the
\(\widetilde{\mathcal O}(T^{-1/4})\) SMG guarantee~\cite{chen2024three}, and a per-iterate CA direction-distance guarantee for
two objectives; see Table~\ref{tab:comparison} for a comparison with existing methods.\\
\textbf{Experiments on MTL applications. }
Comparisons with existing MOL methods demonstrate MoRe's effectiveness in both task performance and convergence behavior.

\begin{table}[t]
\caption{Stochastic MOL comparison. Conv. and CA Dist. denote convergence rate for the squared
Pareto-stationarity measure
and CA direction distance; \(\widetilde{\mathcal O}\) hides log
factors. 
}
\label{tab:comparison}
\vspace{-3mm}
\centering
\footnotesize
\setlength{\tabcolsep}{4pt}
\renewcommand{\arraystretch}{1.1}
\begin{tabular}{lccc}
\toprule
Algorithm & Batch size & Conv. & CA Dist. \\ 
\midrule
SMG~\cite{liu2024stochastic}
\cite[Thms~7,8]{chen2024three} & $\Theta(t+1)$  & \(\widetilde{\mathcal{O}}(T^{-\frac{1}{4}})\)
& per-iterate\\
CR-MOGM~\cite[Thm~3]{zhou2022_SMOO} 
& $\Theta(1)$  & $\mathcal{O}(T^{-\frac{1}{2}})$ 
& - \\
MoCo~\cite[Thm 2, Lem 2]{fernando2022mitigating}
&$\Theta(1)$ &$\mathcal{O}(T^{-\frac{1}{10}})$ &average \\
MoDo~\cite[Thms~3,5]{chen2024three}
&$\Theta(1)$ &${\mathcal{O}}(T^{-\frac{1}{4}})$ &average \\
MoCo+~\cite[Thms 1,2]{fernando2024variance}
&$\Theta(1)$ &$\mathcal{O}(T^{-\frac{2}{3}})$ &average \\
\hline
\bf MoRe, Thms~\ref{thm:conv_main},~\ref{thm:ca} &  $\Theta(t+1)$  & \(\widetilde{\mathcal{O}}(T^{-1})\)
& per-iterate
\\
\bottomrule
\end{tabular}
\vspace{-5mm}
\end{table}

\section{Problem Setup}
Our analysis concerns \(F_S\), but we state standard optimality notions for
the population objective \(F(x)=(f_1(x),\ldots,f_M(x))^\top\), where
\(f_m(x)=\E_z[f_{z,m}(x)]\). They apply empirically by replacing \(F\) with
\(F_S\).

\begin{definition}[Pareto optimal and Pareto stationary points]
A point \(x^*\in\mathbb{R}^p\) is Pareto optimal if there is no
\(x\neq x^*\) such that \(f_m(x)\le f_m(x^*)\) for every \(m\in[M]\),
with strict inequality for at least one objective. It is weakly Pareto
optimal if there is no \(x\in\mathbb{R}^p\) such that
\(f_m(x)<f_m(x^*)\) for every \(m\in[M]\). A point
\(x\in\mathbb{R}^p\) is Pareto stationary if there exists
\(\lambda\in\Delta^M\) such that \(\nabla F(x)\lambda=0\), where
\(\Delta^M\coloneqq\{\lambda\in\mathbb{R}^M:
\mathbf{1}^\top\lambda=1,\ \lambda\ge0\}\) and
\(\nabla F(x)\coloneqq[\nabla f_1(x),\ldots,\nabla f_M(x)]
\in\mathbb{R}^{p\times M}\).

\end{definition}

By definition, a Pareto stationary point admits no common descent direction for all objectives. For smooth objectives, a necessary and sufficient condition for \(x\) to be Pareto stationary~\cite{tanabe2019proximal} is
$
\min_{\lambda\in\Delta^M}\|\nabla F(x)\lambda\| = 0$. We therefore use its squared empirical
counterpart to define the empirical Pareto-stationarity measure~\cite{desideri2012multiple,fernando2022mitigating,chen2024three}:
{
\begin{equation}
R_{S}(x)
\coloneqq
\min_{\lambda\in\Delta^M}
\|\nabla F_S(x)\lambda\|^2 .
\label{eq:r_opt}
\end{equation}
}

\subsection{Multi-gradient descent algorithm (MGDA)}
To solve the multi-objective problem, MGDA~\cite{desideri2012multiple} constructs a CA direction that improves all objectives simultaneously. For a point $x\in\mathbb{R}^p $, MGDA seeks a direction \(d(x)\in\mathbb{R}^p\) that minimizes the
worst-case regularized directional derivative:
\begin{equation}
d(x)
\coloneqq
\arg\min_{d\in\mathbb{R}^{p}}
\max_{m\in[M]}
\left\{
\left\langle \nabla f_{S,m}(x),d\right\rangle
+\frac{1}{2}\|d\|^2
\right\}.
\label{eq:mgda_minmax}
\end{equation}
By minimax duality,~\eqref{eq:mgda_minmax} can be solved by~\cite{mukai1980algorithms,fliege2000steepest}
\begin{align}
&d(x)
=
-\nabla F_S(x)\lambda^*(x)\\
&\text{s.t.}\quad\lambda^*(x)
\in
\mathop{\arg\min}_{\lambda\in\Delta^M}
\left\|
\nabla F_S(x)\lambda
\right\|^2.
\label{eq:mgda_dual}
\end{align}
Here, \(\lambda^*(x)\) denotes the simplex-constrained weight vector used to combine the objective gradients. At iteration \(t\), the model is updated by
\begin{equation}
x_{t+1}
=
x_t+\alpha_t d(x_t),
\label{eq:mgda_update}
\end{equation}
where \(\alpha_t>0\) is the step size. The resulting direction \(d(x_t)\) is a common descent direction whenever \(x_t\) is not Pareto stationary. Moreover, MGDA terminates when \(\nabla F_S(x_t)\lambda^*(x_t)=0\), in which case \(x_t\) is a Pareto stationary point.
\subsection{Stochastic multi-gradient (SMG) method}
A natural stochastic counterpart of MGDA is the SMG method~\cite{liu2024stochastic}, which replaces the full-batch gradients with stochastic gradient estimates computed from sampled data.  
For each sample \(z\in\mathcal Z\), define the vector-valued objective
$F_z(x)\coloneqq
\bigl(f_{z,1}(x),\ldots,f_{z,M}(x)\bigr)^\top$,
and its gradient matrix
$
\nabla F_z(x)
\coloneqq
[\nabla f_{z,1}(x),\ldots,\nabla f_{z,M}(x)]
\in\mathbb R^{p\times M}.
$
Let \(Z=(\zeta_1,\ldots,\zeta_{|Z|})\) be a mini-batch sampled
independently and uniformly with replacement from \(S\). 
We define the
mini-batch gradient matrix as
\begin{equation}
\nabla F_Z(x)
\coloneqq
\frac{1}{|Z|}
\sum_{r=1}^{|Z|}
\nabla F_{\zeta_r}(x).
\end{equation}
For notational simplicity, let \(Q\in\mathbb{R}^{p\times M}\) denote the stochastic gradient matrix \(\nabla F_Z(x)\),
\begin{align}
{
 Q\coloneqq\nabla F_Z(x)=[q_{Z,1}(x),\ldots,q_{Z,M}(x)]\in\mathbb{R}^{p\times M}, 
}
\end{align}
{where \(q_{Z,m}(x)\coloneqq \frac{1}{|Z|}\sum_{r=1}^{|Z|}\nabla f_{\zeta_r,m}(x)\in\mathbb{R}^p\).} Define 
\begin{equation}
    g_Q(\lambda)\coloneqq \|Q\lambda\|^2,
\qquad \lambda\in\Delta^M .\label{eq:g_Q}
\end{equation} 
Then SMG computes the stochastic update direction as
\begin{align}
d_Q= - Q\lambda_Q^*,
\quad \text{s.t.}\quad
\lambda_Q^* \in {\arg\min}_{\lambda\in\Delta^M} g_Q(\lambda),\label{dual_problem}
\end{align}
where \(d_Q\in\mathbb{R}^p\) denotes the update direction associated with \(Q\), and \(\lambda_Q^*\in\Delta^M\) denotes an optimal simplex weight for \(Q\).

To measure how well the update direction tracks the corresponding deterministic MGDA CA direction, we further adopt the CA direction distance introduced in~\cite{chen2024three}. 
The CA direction distance at iteration \(t\) is defined as 
\begin{equation}
\mathcal E_{\mathrm{ca}}(t)
\coloneqq
\left\|\mathbb E[d_t-d(x_t)]\right\|^2, 
~~
d(x_t)=-\nabla F_S(x_t)\lambda^*(x_t),
\end{equation}
where $\lambda^*(x_t)
\in
\mathop{\arg\min}_{\lambda\in\Delta^M}
\left\|
\nabla F_S(x_t)\lambda
\right\|^2$, and $d_t$ is the update direction at iteration \(t\).
The expectation in \(\mathcal E_{\mathrm{ca}}(t)\) is unconditional over the algorithmic randomness. This quantity measures the discrepancy between the stochastic update direction \(d_t\) and the full-batch CA direction \(d(x_t)\), thereby quantifying the conflict-avoidance error of
the stochastic update.

\section{Regularity-Aware Stochastic MGDA}
   
In this section, we first characterize the CA direction mapping across different geometric regimes, then develop MoRe, a regularity-aware aggregation method for MOL, and establish its convergence and conflict-avoidance guarantees.
  
\subsection{Continuity properties of the CA direction mapping}
Although the optimizer \(\lambda_Q^*\) in~\eqref{dual_problem} may be
nonunique, \(-d_Q=Q\lambda_Q^*\) is unique as the minimum-norm point in the convex hull of
\(\{q_{Z,1}(x),\ldots,q_{Z,M}(x)\}\). We show that
\(Q\mapsto d_Q\) is Lipschitz continuous under nondegeneracy but only
\(1/2\)-H\"older continuous on bounded sets in general.

\begin{definition}
[\(\mu\)-nondegeneracy]\label{def:pl_main}
Let \(P:=I-\frac1M\mathbf 1\mathbf 1^\top\) be the orthogonal projector onto
\(\mathbf 1^\perp:=\{v\in\mathbb R^M:\mathbf 1^\top v=0\}\). Let
\(U\in\mathbb R^{M\times(M-1)}\) be any matrix whose columns form an
orthonormal basis of \(\mathbf 1^\perp\), so that \(U^\top U=I\) and
\(UU^\top=P\). 
We say a matrix \(Q\in\mathbb R^{p\times M}\) is
\(\mu\)-nondegenerate if
\begin{equation}
\mu_{\min}\!\left(U^\top Q^\top Q U\right)\ge \mu>0,
\label{eq:nondegeneracy_tangent}
\end{equation}
where \(\mu_{\min}(\cdot)\) denotes the smallest eigenvalue. Equivalently, \(\|Qv\|^2\ge \mu\|v\|^2\) for every
\(v\in\mathbf 1^\perp\).
\end{definition}
Here
$
\rho(Q)\coloneqq\mu_{\min}(U^\top Q^\top Q U)
$
is the curvature of \(g_Q(\lambda)=\|Q\lambda\|^2\) along the simplex
tangent space. The condition \(\rho(Q)\ge\mu\) holds when the
gradients are quantitatively separated from affine dependence and ensures
a unique MGDA weight.
For \(M=2\),
\(\rho(Q_t)=\frac12\|q_{Z_t,1}-q_{Z_t,2}\|^2\); hence, near-degeneracy occurs when the
two gradients are more identical. 

\begin{lemma}[Lipschitz continuity of $d_Q$ w.r.t. $Q$]\label{lem:lip_dq}
Let \(Q_1,Q_2\in\mathbb{R}^{p\times M}\) satisfy
\(\|Q_i\|_F\le B\) and
\(\mu_{\min}(U^\top Q_i^\top Q_iU)\ge\mu>0\) for \(i=1,2\), i.e., they are
\(\mu\)-nondegenerate as in Definition~\ref{def:pl_main}. Then
\(Q\mapsto d_Q\) is Lipschitz continuous with
\(\ell_d=1+2B^2/\mu\):
\begin{equation}
\|d_{Q_1}-d_{Q_2}\|
\le
\ell_d\|Q_1-Q_2\|.
\end{equation}
Here \(\|\cdot\|\) is the spectral norm; since
\(\|A\|\le\|A\|_F\), the bound also holds in Frobenius norm on the right-hand side.
\end{lemma}
Lemma~\ref{lem:lip_dq} characterizes the nondegenerate regime of the MGDA
subproblem. It shows that, on any bounded set of matrices satisfying Definition~\ref{def:pl_main}, the CA direction depends Lipschitz continuously on the gradient matrix. 
\begin{lemma}[Sharpness of the \(1/2\)-H\"older exponent] 
Let \(Q_1,Q_2\in\mathbb{R}^{p\times M}\) satisfy
\(\|Q_i\|_F\le B\) for \(i=1,2\). The mapping \(Q\mapsto d_Q\) is \(1/2\)-H\"older continuous on this bounded set:
\begin{equation}
\|d_{Q_1}-d_{Q_2}\|
\le
2\sqrt{B}\,\|Q_1-Q_2\|^{1/2}.    
\end{equation}
Moreover,
without additional assumptions, no H\"older continuity with exponent
\(\eta>\tfrac{1}{2}\) holds in general.
\label{lem:holder_sharp}
\end{lemma}

Lemmas~\ref{lem:lip_dq} and~\ref{lem:holder_sharp} identify reduced curvature as a geometric indicator: a uniform positive lower bound yields Lipschitz continuity, whereas the worst-case class admits only \(1/2\)-H\"older continuity. This motivates the switching rule below.

\begin{algorithm}[t]
\caption{MoRe: Multi-objective Regularity-aware method}\label{alg:alg}
\begin{algorithmic}[1]
\State \textbf{input} 
training data \(S=(z_1,\ldots,z_n)\), initial model \(x_0\), step sizes \(\{\alpha_t\}_{t=0}^{T-1}\), batch sizes \(\{|Z_t|\}_{t=0}^{T-1}\), thresholds \(\{\mu_t\}_{t=0}^{T-1}\), fallback weight \(\bar\lambda\in\Delta^M\), and an orthonormal basis \(U\in\mathbb R^{M\times(M-1)}\) of \(\mathbf 1^\perp\).
\For { $t=0, \dots, T-1$}
\State {Sample \(Z_t\) i.i.d. uniformly from \(S\)}
\State Compute $Q_t = [q_{Z_t,1}(x_t), \dots, q_{Z_t,M}(x_t)]$
\If { $\mu_{\min}(U^\top Q_t^\top Q_tU) \ge \mu_t$} 
\State Update $\lambda_t$ by \eqref{eq: lambda_update}
\Else 
\State {Set $\lambda_{t}=\bar\lambda$}
\EndIf
\State Update $x_{t+1} $ by \eqref{eq: x_update}
\EndFor
\State \textbf{output} $x_T$.
\end{algorithmic}
\end{algorithm}
\begin{figure}
\centering
\includegraphics[width=0.98\linewidth]{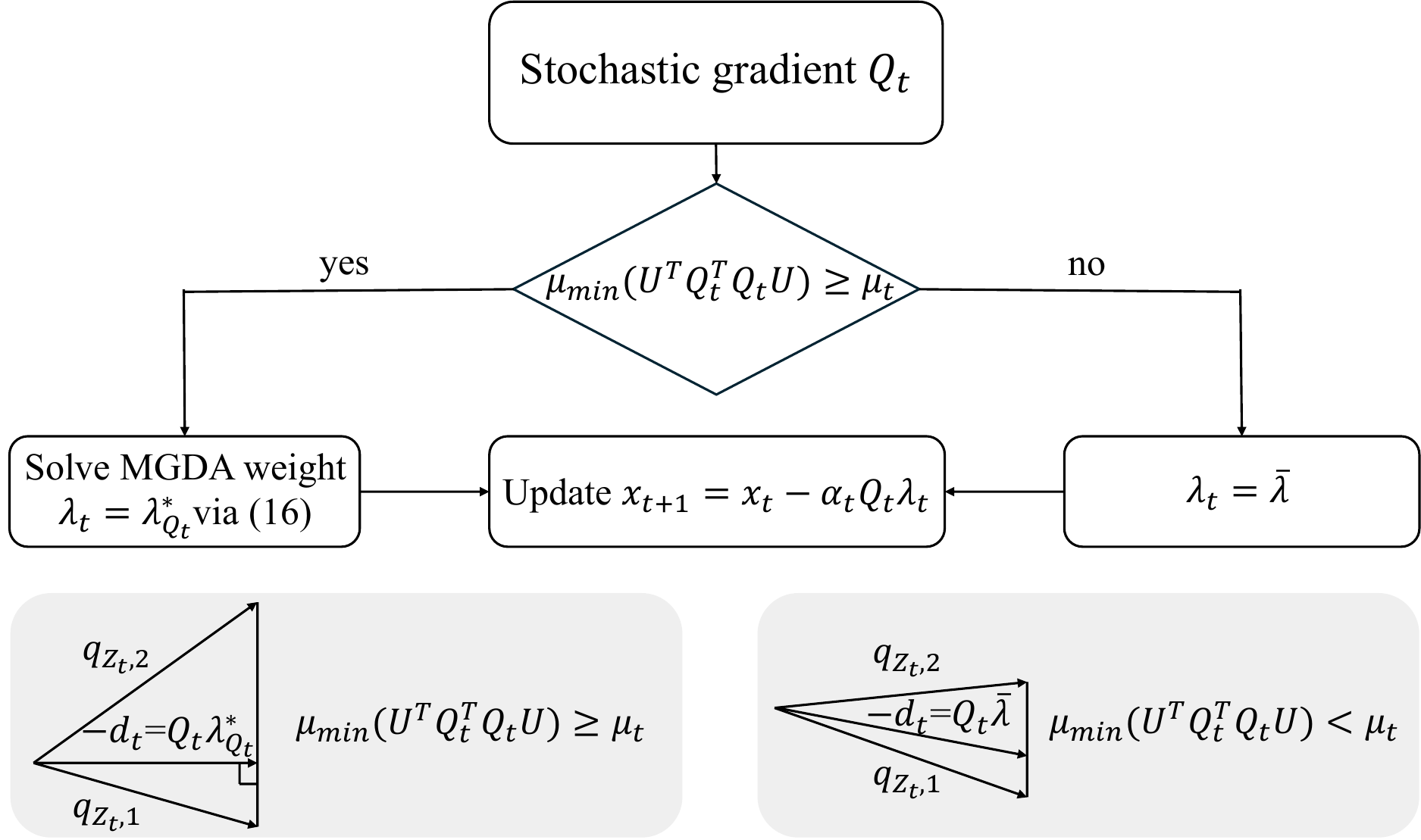}
\caption{Flowchart of the proposed MoRe algorithm. 
The illustrations at the bottom show the switching rule in the two-objective case, where
\(\mu_{\min}(U^\top Q_t^\top Q_t U)=\frac{1}{2}\|q_{Z_t,1}-q_{Z_t,2}\|^2\) measures the disagreement
between the two stochastic gradients. 
When \(\mu_{\min}(U^\top Q_t^\top Q_t U)\ge \mu_t\), MoRe uses the MGDA CA update. Otherwise, MoRe treats the stochastic gradient matrix as nearly degenerate, skips the solve, and applies the fixed linear scalarization (LS) direction
\(d_t=-Q_t\bar\lambda\).
}
\vspace{-5mm}
\label{fig:flow}
\end{figure}
\subsection{Multi-objective regularity-aware method}
Building on the two-regime characterization, we propose the Multi-objective Regularity-aware (MoRe) method. The geometry of the MGDA subproblem
may vary along the optimization trajectory, yet standard stochastic MGDA solves the aggregation subproblem at every iteration, regardless of its regularity. MoRe instead uses the reduced curvature defined in Definition~\ref{def:pl_main} as a computable regularity indicator for selecting the aggregation rule. As illustrated in
Figure~\ref{fig:flow}, MoRe performs MGDA aggregation in the regular
regime to preserve conflict avoidance. Near degeneracy, where the MGDA direction may be sensitive to gradient perturbations, it bypasses the MGDA solve and uses a prescribed linear scalarization (LS) fallback. Thus, MoRe retains CA updates in the regular regime while reducing aggregation cost on fallback iterations.

Specifically, at iteration \(t\), MoRe samples \(Z_t\) independently with replacement and
forms \(Q_t=[q_{Z_t,1}(x_t),\ldots,q_{Z_t,M}(x_t)]\). If
\(\mu_{\min}(U^\top Q_t^\top Q_tU)\ge\mu_t>0\), it computes

\begin{equation}
\lambda_t=\lambda_{Q_t}^*
\in
\mathop{\arg\min}_{\lambda\in\Delta^M}\|Q_t\lambda\|^2.\label{eq: lambda_update}
\end{equation}
Otherwise, MoRe sets the fixed simplex weight
\(\lambda_t=\bar\lambda\in\Delta^M\) without solving
\eqref{eq: lambda_update}. The fallback weight \(\bar\lambda\) can be chosen as any prescribed simplex vector reflecting the desired scalarization; for example, in our experiments, we use the uniform choice \(\bar\lambda=\frac{1}{M}\mathbf 1\). Finally, both branches update
\begin{equation}
x_{t+1}=x_t-\alpha_t Q_t\lambda_t.\label{eq: x_update}
\end{equation}

\subsection{Theoretical analysis}
We present the standard assumptions and theoretical guarantees for
MoRe; all proofs are available at
\href{https://arxiv.org/pdf/2607.15412}{arXiv:2607.15412}.

\begin{assumption}[Finite lower bounds]\label{ass:finite_lower_bound}
For each $m\in[M]$, $
\inf_x f_{S,m}(x)>-\infty.$
\end{assumption}
Assumption~\ref{ass:finite_lower_bound} is standard and ensures that the LS objective \(\lambda^\top F_S\) is bounded below for every fixed \(\lambda\in\Delta^M\).
\begin{assumption}[Unbiasedness and bounded variance]
\label{ass:unbiased_main}
Let \(\mathcal F_t\) denote the algorithmic history prior to sampling \(Z_t\). Conditional on \(\mathcal F_t\), \(Z_t=(\zeta_{t,r})_{r=1}^{|Z_t|}\) is sampled i.i.d. uniformly from \(S\). There exists a constant \(\sigma>0\)
such that, for all \(t\)
and \(m\in[M]\), almost surely,
\begin{align}
&\mathbb E\!\left[q_{Z_t,m}(x_t)\mid\mathcal F_t\right]
=\nabla f_{S,m}(x_t),\\
&\mathbb E\!\left[
\left\|q_{Z_t,m}(x_t)-\nabla f_{S,m}(x_t)\right\|^2
\mid\mathcal F_t
\right]
\le \frac{\sigma^2}{|Z_t|}.
\end{align}
\end{assumption}
Assumption~\ref{ass:unbiased_main} requires the stochastic gradients to be unbiased and to have bounded variance. 

\begin{assumption}[Smoothness]\label{ass:smooth_main}
There exists a constant \(\ell_{f,1}>0\) such that, for every \(\lambda\in\Delta^M\), the LS empirical objective \(x\mapsto \lambda^\top F_S(x)\) is \(\ell_{f,1}\)-smooth on \(\mathbb R^p\).
\end{assumption}
\begin{assumption}[Bounded gradients]
\label{ass:gradbound_main}
There exist finite positive constants \(\ell_f\) and \(\ell_F=\sqrt{M}\ell_f\) such that, for every \(x\in\mathbb R^p\), \(z\in\mathcal Z\), and \(m\in[M]\), $\|\nabla f_{z,m}(x)\|\le \ell_f$. Since \(q_{Z,m}(x)=|Z|^{-1}\sum_{r=1}^{|Z|}\nabla f_{\zeta_r,m}(x)\), the triangle inequality gives \(\|q_{Z,m}(x)\|\le \ell_f\). Therefore, for every mini-batch \(Z\), every \(t\), and every \(\lambda\in\Delta^M\),
\begin{align}
\|Q_t\lambda\| &\le \ell_f, \quad
\|\nabla F_S(x_t)\lambda\| \le \ell_f, \notag\\
\|Q_t\|_F &\le \ell_F, \quad
\|\nabla F_S(x_t)\|_F \le \ell_F .
\label{eq:uniform_sample_gradient_consequences}
\end{align}
\end{assumption}
Assumptions~\ref{ass:smooth_main} and~\ref{ass:gradbound_main} are also standard in gradient-based MOL analysis~\cite{chen2024three}.
\begin{theorem}[Empirical Pareto-stationarity rate]
\label{thm:conv_main}
Suppose Assumptions~\ref{ass:finite_lower_bound},~\ref{ass:unbiased_main}, and~\ref{ass:smooth_main} hold. Let
\(\alpha_t=\alpha\in(0,1/(2\ell_{f,1})]\) and \(\mu_t>0\) for all \(t\). Define
\(
\Delta_0\coloneqq \bar\lambda^\top F_S(x_0)-
\inf_x\bar\lambda^\top F_S(x).
\)
Then the iterates generated by
Algorithm~\ref{alg:alg} satisfy
\begin{equation}
\mathbb E\!\left[
{\min_{0\le t\le T-1}}
R_{S}(x_t)
\right]
\le
{\frac{8\Delta_0}{\alpha T}}
+\frac{6M\sigma^2}{T}
\sum_{t=0}^{T-1}\frac{1}{|Z_t|}.
\label{eq:thm_general_rate}
\end{equation}
In particular, if \(|Z_t|=\Theta(t+1)\), then
\begin{equation}
\mathbb E\!\left[
{\min_{0\le t\le T-1}}
R_{S}(x_t)
\right]
=\widetilde{\mathcal O}(T^{-1}).
\label{eq:thm_final_rate_tilde}
\end{equation}
\end{theorem}
With linearly growing mini-batches and a constant step size,
Theorem~\ref{thm:conv_main} yields
\(\widetilde{\mathcal O}(T^{-1})\), improving upon the
\(\widetilde{\mathcal O}(T^{-1/4})\) SMG guarantee in
Table~\ref{tab:comparison}. This rate is independent of \(\{\mu_t\}\).

\begin{figure}[t]
\centering
\includegraphics[width=0.95\linewidth]{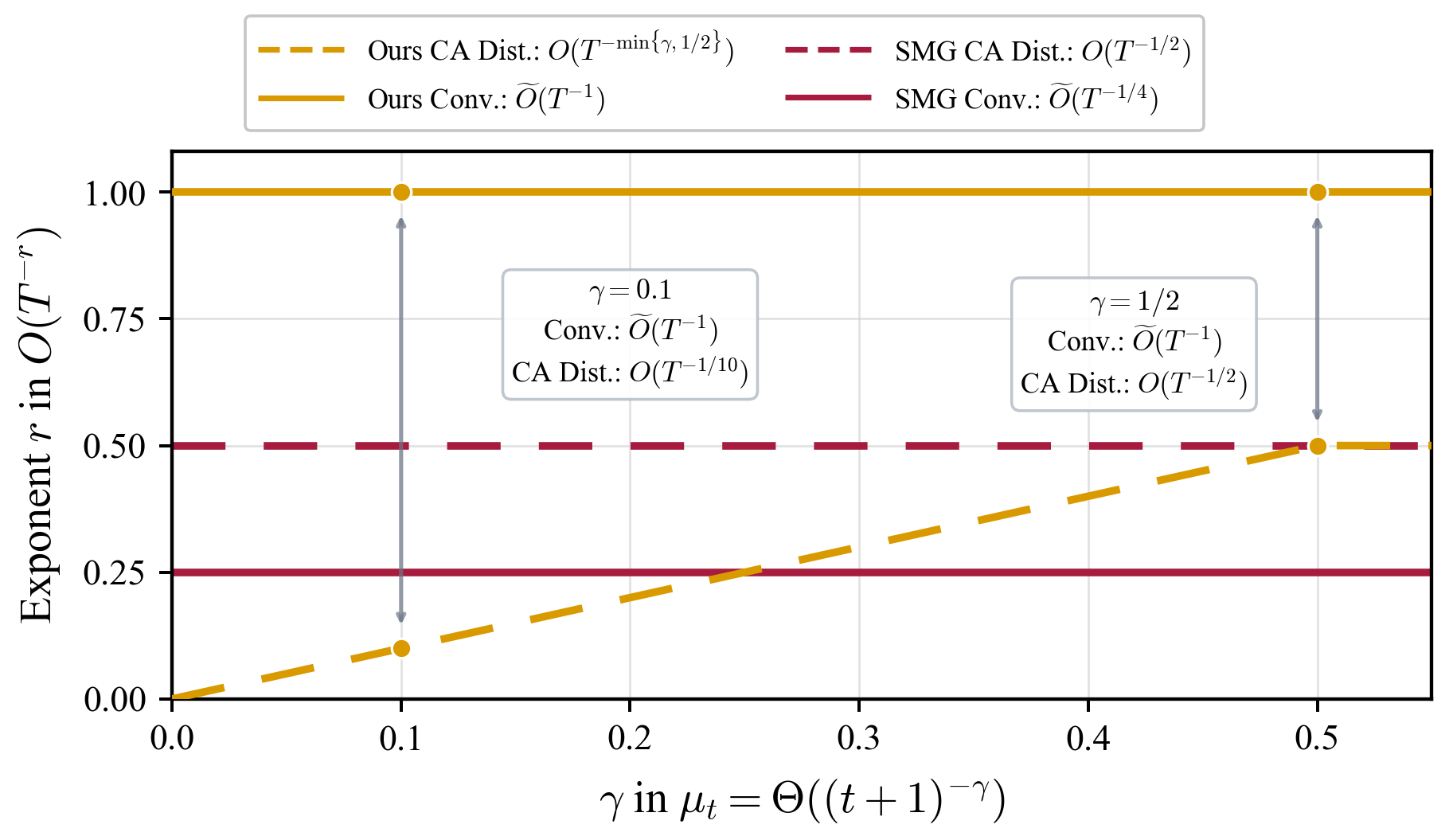}
\vspace{-3mm}
\caption{
Comparison of rates for \(\mu_t=\Theta((t+1)^{-\gamma})\), and \(|Z_t|=\Theta(t+1)\). MoRe achieves
\(\widetilde{\mathcal O}(T^{-1})\) convergence and
\(\mathcal O(T^{-\min\{\gamma,1/2\}})\) CA rate; the corresponding SMG convergence and CA rates are
\(\widetilde{\mathcal O}(T^{-1/4})\) and \(\mathcal O(T^{-1/2})\),
respectively. The plotted lines omit logarithmic
factors.
}
\label{fig:rate_tradeoff}
\vspace{-5mm}
\end{figure}
\begin{theorem}[Per-iterate CA direction distance]\label{thm:ca}
Suppose Assumptions~\ref{ass:unbiased_main} and~\ref{ass:gradbound_main} hold, and let \(M=2\). Let \(\mu_t>0\). Then the CA direction distance satisfies
\begin{equation}
\mathcal E_{\mathrm{ca}}(t)
\le
\frac{4\sqrt{2}\,\ell_F\sigma}{\sqrt{|Z_t|}}
+
\frac{6\sigma^2}{|Z_t|}
+
6\mu_t.
\label{eq:thm_ca_main}
\end{equation}
If \(|Z_t|=\Theta(t+1)\) and
\(\mu_t=\Theta((t+1)^{-1/2})\), then
\begin{equation}
\mathcal E_{\mathrm{ca}}(t)
=
\mathcal O((t+1)^{-1/2}).
\end{equation}
\end{theorem}

For \(M=2\), 
the regularity test bounds the fallback error through \(\mu_t\). Thus, MoRe can
skip MGDA solves on fallback iterations while retaining the per-iterate rate
\(\mathcal O((t+1)^{-1/2})\) under the schedules in
Theorem~\ref{thm:ca}.

Figure~\ref{fig:rate_tradeoff} shows that, under the schedules in
Theorem~\ref{thm:conv_main}, for \(\mu_t=\Theta((t+1)^{-\gamma})\), the convergence rate remains
\(\widetilde{\mathcal O}(T^{-1})\), while the CA exponent increases with
\(\gamma\) up to \(1/2\). Smaller \(\gamma\) may trigger more fallback
updates and save MGDA solves at the cost of a looser CA bound; the SMG
baselines are independent of \(\mu_t\).

\section{Experiment}
Following the experimental setting in~\cite{chen2024three}, we use LibMTL~\cite{lin2023libmtl} to
evaluate MoRe against representative baselines on Office-Home~\cite{venkateswara2017deep}. The dataset contains four domains, each comprising 65 classes: \emph{Art}, \emph{Clipart}, \emph{Product}, and \emph{Real-World}. We report
individual task performance and the average relative degradation:
\begin{equation}
\Delta_{A}^{\mathrm{id}}\%
=
\frac{1}{M}\sum_{m=1}^{M}
(-1)^{\ell_m}\frac{S_{A,m}-S_{B,m}}{S_{B,m}}\times 100,
\end{equation}
where \(S_{A,m}\) and \(S_{B,m}\) denote the performance of MOL method \(A\) and that of the single-task baseline on task \(m\), respectively. \(\ell_m=1\) for higher-is-better metrics and \(0\)
otherwise; lower $\Delta_{A}^{\mathrm{id}}\%$ indicates better overall performance. The four single-task baseline accuracies are \(66.86\%\),
\(82.02\%\), \(91.38\%\), and \(81.24\%\), respectively. Because these baselines were obtained under our setting, the degradation
values may differ from prior work. 

Table~\ref{tab:officehome} compares MoRe with CAGrad, deterministic
MGDA, fixed LS, SMG, MoCo, and MoDo. We evaluate MoRe under two
batch-size schedules: a fixed batch size and a growing
batch size. The fixed schedule uses \(|Z_t|=32\), whereas
the growing schedule uses
$
|Z_t|=\min\{b_{\max},\,b_0+\lfloor ct\rfloor\}$,
where \(b_0\), \(c\), and \(b_{\max}\) denote the initial batch size,
per-iteration growth rate, and maximum batch size, respectively. MoRe with the growing schedule
improves its fixed-batch counterpart on all four domains and reduces
\(\Delta_A^{\mathrm{id}}\) from \(3.07\%\) to \(1.94\%\). Fixed
LS and SMG serve as the two no-switch ablations: the former always uses
\(\bar\lambda\), whereas the latter solves the MGDA subproblem
at every iteration. Under both batch-size schedules, MoRe achieves lower degradation than these two ablations, supporting
the benefit of adaptive switching. 
With the growing schedule, MoRe achieves the highest accuracy on
three domains and the lowest overall degradation $\Delta_A^{\mathrm{id}}\%$. Although it slightly sacrifices Product accuracy, MoRe provides the best overall balance across
domains.

Figure~\ref{fig:convergence_curve} reports the
running-best empirical Pareto-stationarity measure,
\(\min_{0\le i\le t-1}R_S(x_i)\), evaluated with full-batch training gradients.
With a step size \(\alpha_t=\alpha\in(0,1/(2\ell_{f,1})]\) and batch size \(|Z_t|=\Theta(t+1)\),  Theorem~\ref{thm:conv_main} gives an expected rate of
\(\widetilde{\mathcal O}(T^{-1})\). The observed decay is consistent with
this rate, up to logarithmic factors; the dashed \(c/t\) line serves as a visual reference.

\begin{table}[t]
\caption{Office-Home test accuracy (mean $\pm$ standard deviation).}
\label{tab:officehome}
\centering
\setlength{\tabcolsep}{2pt}
\renewcommand{\arraystretch}{1.08}
\resizebox{\columnwidth}{!}{%
\begin{tabular}{lccccc}
\toprule
Method & Art & Clipart & Product & Real-World &
$\Delta_{A}^{\mathrm{id}}\% \downarrow$ \\
\midrule
CAGrad~\cite{liu2021conflict}
& $63.41 \pm 0.49$
& $76.49 \pm 0.77$
& $89.35 \pm 0.38$
& $78.90 \pm 0.89$ 
& $4.25 \pm 0.43$\\

Deterministic MGDA
& $59.01 \pm 0.50$
& $73.96 \pm 1.02$
& $87.04 \pm 1.91$
& $78.63 \pm 1.27$
& $7.38 \pm 0.88$ \\
Fixed LS
& $64.90 \pm 3.75$
& $76.92 \pm 1.99$
& $86.18 \pm 2.77$
& $78.27 \pm 2.29$
& $4.63 \pm 3.47$\\
SMG~\cite{liu2024stochastic}
& $65.69 \pm 1.40$
& $74.89 \pm 1.82$
& $87.56 \pm 0.55$
& $78.75 \pm 0.72$
& $4.42 \pm 0.69$ \\

MoCo~\cite{fernando2022mitigating}
& $63.69 \pm 0.61$
& $79.74 \pm 0.29$
& $88.28 \pm 1.17$
& $79.14 \pm 0.39$
& $3.37 \pm 0.69$ \\

MoDo~\cite{chen2024three}
& $65.50 \pm 0.55$
& $79.44 \pm 0.29$
& $\mathbf{89.72 \pm 0.94}$
& $79.65 \pm 0.67$
& $2.24 \pm 0.48$ \\
MoRe (fixed batch)
& $64.96 \pm 0.77$
& $78.66 \pm 0.96$
& $87.61 \pm 0.76$
& $80.25 \pm 0.53$
& $3.07 \pm 0.39$ \\
MoRe (growing batch)
& $\mathbf{66.41 \pm 0.66}$
& $\mathbf{79.92 \pm 0.22}$
& $87.71 \pm 0.64$
& $\mathbf{80.83 \pm 0.88}$
& $\mathbf{1.94 \pm 0.69}$ \\
\bottomrule
\end{tabular}%
}
\end{table}

\begin{figure}[t]
\centering
\includegraphics[width=0.95\linewidth]{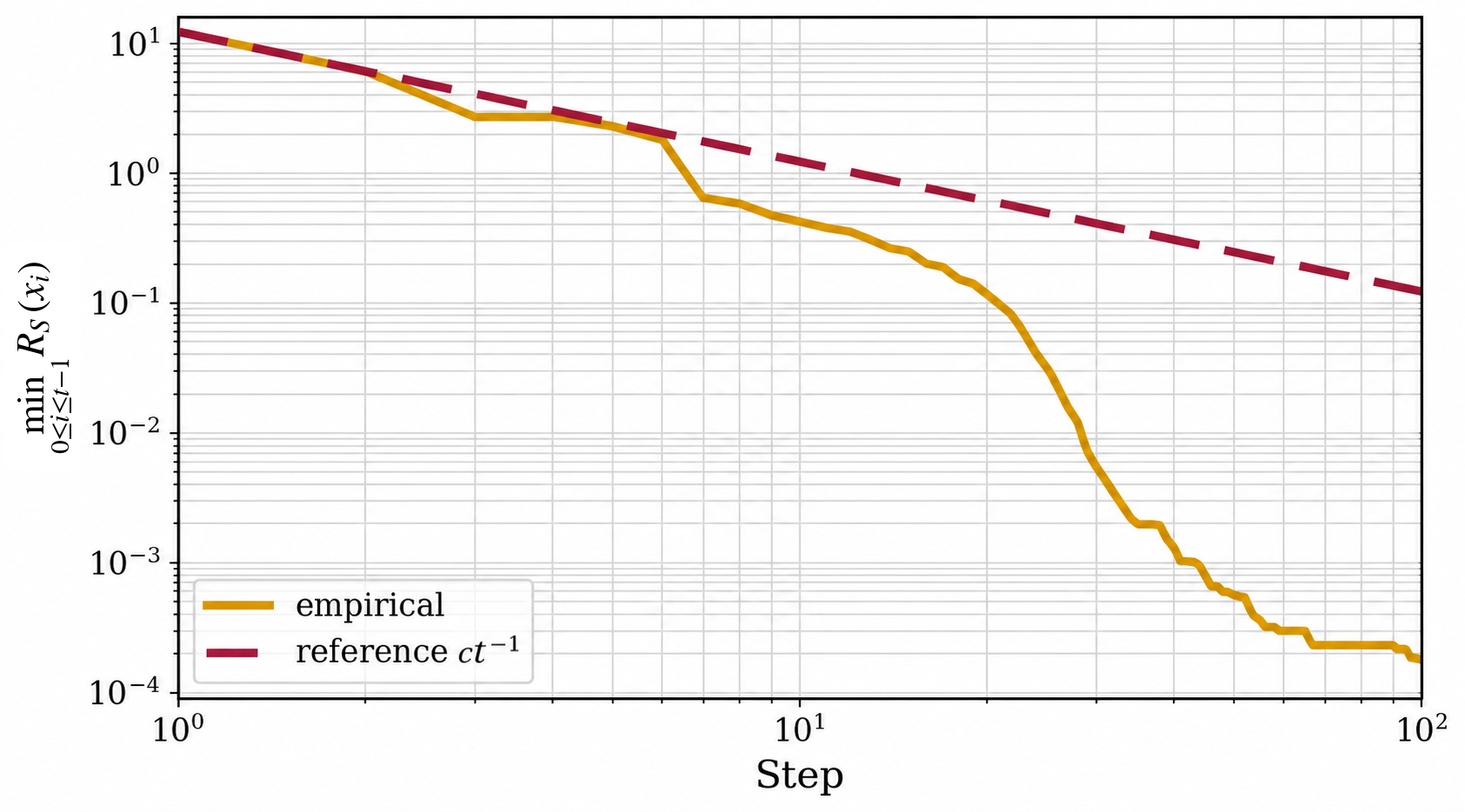}
\vspace{-3mm}
\caption{Convergence curve of MoRe on Office-Home. The solid
curve shows the running-best squared empirical Pareto-stationarity measure
\(\min_{0\le i\le t-1}R_S(x_i)\), and the dashed line shows the
\(ct^{-1}\) reference slope. 
}
\label{fig:convergence_curve}
\vspace{-5mm}
\end{figure}
\section{Conclusion}
This paper characterizes the regularity of the MGDA direction mapping,
establishing Lipschitz continuity under reduced nondegeneracy and
proving that the \(1/2\)-H\"older exponent is sharp in general.
Guided by this characterization, MoRe applies the MGDA update
in the regular regime to preserve conflict avoidance. Near degeneracy,
it uses a fixed LS weight to avoid potentially noise-sensitive weights and skips the MGDA solve to reduce computational cost.
We establish a \(\widetilde{\mathcal O}(T^{-1})\) convergence rate,
improving upon the \(\widetilde{\mathcal O}(T^{-1/4})\) SMG guarantee,
together with a per-iterate CA direction-distance guarantee for two
objectives. 
Office-Home experiments
show favorable multi-task performance and convergence consistent with the
theory. Extending the per-iterate CA analysis to the general multi-objective setting with $M>2$ is left for future work.

\bibliographystyle{IEEEtran}
\bibliography{IEEEabrv,myabrv,reference,MOO_opt}

\onecolumn
\newpage

\appendices
\section{Notation}
A summary of notation used in this work is listed in Table~\ref{tab:notation} for ease of reference.
\begin{table}[ht]
\caption{{Notation and descriptions.}}
  \label{tab:notation}
  \small
  \centering
  \begin{tabular}{l|l }
  \toprule
  Notation & Descriptions   \\
  \midrule
   $x \in \mathbb{R}^{p}$ &Model parameter or decision variable   \\
   \makecell[l]{\(z\in\mathcal Z\), \(S\), \(Z\)} 
&
\makecell[l]{{Data point, training sample sequence, and mini-batch, respectively, where} \\
{\(S=(z_1,\ldots,z_n)\in\mathcal Z^n\),} \\
{\(Z=(\zeta_1,\ldots,\zeta_{|Z|})\) is sampled  with replacement from \(S\).}} \\
  \makecell[l]{$f_{z,m}(x)$, $f_{S,m}(x)$ \\
  } &
    \makecell[l]{Scalar-valued objective functions evaluated on data point $z$,\\ with $f_{z,m}: \mathbb{R}^p \to \mathbb{R}$, 
  or on dataset $S$, $f_{S,m}$, 
  with $f_{S,m}(x) \coloneqq \frac{1}{n} \sum_{i=1}^n f_{z_i,m}(x)$, respectively} \\
  \makecell[l]{$F_{z}(x)$, $F_{S}(x)$, $F_Z(x)$\\
  } &
  \makecell[l]{Vector-valued objective functions evaluated on data point $z$,\\ with $F_{z}(x): \mathbb{R}^p \to \mathbb{R}^M$,
  or on dataset $S$, with
  {$F_{S}(x) \coloneqq \frac{1}{n} \sum_{i=1}^n F_{z_i}(x)$} \\
{or on mini-batch $Z$, with}
  {$F_{Z}(x) \coloneqq \frac{1}{|Z|} \sum_{r=1}^{|Z|} F_{\zeta_r}(x)$, respectively}
  } \\
{\(\nabla F_S(x)\), \(\nabla F_Z(x)\)}
&
Gradient matrices of the empirical objective \(F_S(x)\) and the mini-batch 
objective \(F_Z(x)\), respectively\\
  $f_{m}(x)$ &A scalar-valued population objective function, $f_{m}(x) \coloneqq \E_z [f_{z,m}(x)] $ \\
  $\nabla f_{m}(x)$ &Gradient of $f_{m}(x)$, with 
  $\nabla f_{m}(x): \mathbb{R}^p \to \mathbb{R}^p$ \\
   $F(x)$ & A vector-valued population objective, $F(x) \coloneqq \E_z [F_{z}(x)] $ \\
  $\nabla F(x)$ &Gradient of $F(x)$, with 
  $\nabla F(x): \mathbb{R}^p \to \mathbb{R}^{p \times M}$ \\
\makecell[l]{\(Q\), \(q_{Z,m}(x)\)}
&
\makecell[l]{Stochastic gradient matrix and its \(m\)-th column, respectively, where \\
\(Q=\nabla F_Z(x)=[q_{Z,1}(x),\ldots,q_{Z,M}(x)]\in\mathbb{R}^{p\times M}\) \\
{\(q_{Z,m}(x)=\frac{1}{|Z|}\sum_{r=1}^{|Z|}\nabla f_{\zeta_r,m}(x)\in\mathbb{R}^p\)}} \\
 $\lambda \in \Delta^{M} $ &Weighting vector in the $(M-1)$-simplex  \\

  $\lambda^*(x) \in \Delta^{M} $ &CA weight, optimal solution to~\eqref{eq:mgda_dual}   \\
  \(\lambda_Q^*\in \Delta^{M}\)
& Optimal simplex weight for \(Q\), solution to~\eqref{dual_problem} \\
\makecell[l]{\(d_Q\), \(d(x_t)\), \(d_t\)}

&

\makecell[l]{
\(d_Q=-Q\lambda_Q^*\) is the CA direction associated with \(Q\), \\
\(d(x_t)=-\nabla F_S(x_t)\lambda^*(x_t)\) is the full-batch CA direction at \(x_t\), \\
and \(d_t=-Q_t\lambda_t\) is the stochastic update direction used at iteration \(t\).} \\
\(\mu_t\)
&
Regularity threshold used by Algorithm~\ref{alg:alg} to decide whether the stochastic subproblem is nondegenerate.  \\
  $\alpha_t$ &Step size for updating the model parameter $x$\\
  \bottomrule
    \end{tabular}
  \vspace{0.2cm}
\end{table}
\section{Proof}
\subsection{Proof of Lemma~\ref{lem:lip_dq}}\label{app:lip_dq}
\begin{proof}
Throughout this proof, \(\|\cdot\|\) denotes the spectral
norm for matrices and the Euclidean norm for vectors.
By the boundedness condition in Lemma~\ref{lem:lip_dq} and the inequality \(\|Q_i\|\le \|Q_i\|_F\), we have \(\|Q_i\|\le B\), \(i=1,2\).
Let \(\lambda_{Q_i}^*\in\mathop{\arg\min}_{\lambda\in\Delta^M}\|Q_i\lambda\|^2\), \(A_i\coloneqq Q_i^\top Q_i\), \(i=1,2\). By the reduced nondegeneracy condition in Definition~\ref{def:pl_main}, \(\mu_{\min}(U^\top Q_i^\top Q_iU)\ge \mu\). Hence, for any \(\delta\in\mathbf 1^\perp\), writing \(\delta=Uz\), we have
\begin{equation}
\delta^\top A_i\delta
=
z^\top U^\top A_iUz
\ge
\mu\|z\|^2
=
\mu\|\delta\|^2.
\label{eq:app_reduced_sc}
\end{equation}
Since any two points in \(\Delta^M\) differ by a vector in \(\mathbf 1^\perp\), \(g_{Q_i}(\lambda)=\|Q_i\lambda\|^2\) is strongly convex on the affine hull of the simplex. Hence \(\lambda_{Q_i}^*\) is unique. The optimality conditions give \(0\in 2A_i\lambda_{Q_i}^*+N_{\Delta^M}(\lambda_{Q_i}^*)\), or equivalently,
\begin{equation}
\langle 2A_i\lambda_{Q_i}^*,\lambda-\lambda_{Q_i}^*\rangle\ge 0,
\qquad
\forall \lambda\in\Delta^M.
\label{eq:app_lip_opt_cond}
\end{equation}
Let \(\delta=\lambda^*_{Q_1}-\lambda^*_{Q_2}\). Taking \(\lambda=\lambda^*_{Q_2}\) for \(i=1\) and \(\lambda=\lambda^*_{Q_1}\) for \(i=2\) in \eqref{eq:app_lip_opt_cond} gives \(\langle A_1\lambda^*_{Q_1},\delta\rangle\le0\) and \(\langle A_2\lambda^*_{Q_2},\delta\rangle\ge0\), hence
\(\langle A_1\lambda^*_{Q_1}-A_2\lambda^*_{Q_2},\delta\rangle\le0\). Since
\(A_1\lambda^*_{Q_1}-A_2\lambda^*_{Q_2}=A_1\delta+(A_1-A_2)\lambda^*_{Q_2}\), we have
\(\delta^\top A_1\delta+\langle (A_1-A_2)\lambda^*_{Q_2},\delta\rangle\le0\). Rearranging this inequality and applying Cauchy--Schwarz gives
\begin{align}
\delta^\top A_1\delta
\le
-\langle (A_1-A_2)\lambda^*_{Q_2},\delta\rangle 
\le
\|A_1-A_2\|\,\|\lambda^*_{Q_2}\|\,\|\delta\|.
\label{eq:app_lip_delta_pre}
\end{align}
Because \(\lambda^*_{Q_2}\in\Delta^M\), \(\|\lambda^*_{Q_2}\|\le 1\). Also, \(\delta\in\mathbf 1^\perp\), so \eqref{eq:app_reduced_sc} gives \(\delta^\top A_1\delta\ge \mu\|\delta\|^2\). Therefore, combining it with \eqref{eq:app_lip_delta_pre} yields
\begin{equation}
\|\delta\|=\|\lambda^*_{Q_1}-\lambda^*_{Q_2}\|
\le
\frac{1}{\mu}\|A_1-A_2\|.
\label{eq:app_lip_lambda_A}
\end{equation}
Moreover, \(A_1-A_2=Q_1^\top(Q_1-Q_2)+(Q_1-Q_2)^\top Q_2\), and hence
\begin{equation}
\|A_1-A_2\|
\le
(\|Q_1\|+\|Q_2\|)\|Q_1-Q_2\|
\le
2B\|Q_1-Q_2\|.
\label{eq:app_lip_A_diff}
\end{equation}
Combining \eqref{eq:app_lip_lambda_A} and \eqref{eq:app_lip_A_diff}, we get
\begin{equation}
\|\lambda^*_{Q_1}-\lambda^*_{Q_2}\|
\le
\frac{2B}{\mu}\|Q_1-Q_2\|.
\label{eq:app_lip_lambda_diff}
\end{equation}
Finally, since \(d_Q=-Q\lambda^*_Q\), we have \(d_{Q_1}-d_{Q_2}=-Q_1(\lambda^*_{Q_1}-\lambda^*_{Q_2})-(Q_1-Q_2)\lambda^*_{Q_2}\). Using \(\|Q_1\|\le B\), \(\|\lambda^*_{Q_2}\|\le 1\), and \eqref{eq:app_lip_lambda_diff}, we obtain
\begin{align}
\|d_{Q_1}-d_{Q_2}\|
\le
\|Q_1\|\,\|\lambda^*_{Q_1}-\lambda^*_{Q_2}\|
+
\|Q_1-Q_2\|\,\|\lambda^*_{Q_2}\| 
\le
B\|\lambda^*_{Q_1}-\lambda^*_{Q_2}\|+\|Q_1-Q_2\| 
\le
\left(1+\frac{2B^2}{\mu}\right)\|Q_1-Q_2\|.
\end{align}
This proves the claim.

\end{proof}
\subsection{Proof of Lemma~\ref{lem:holder_sharp}}\label{app:holder_sharp}
\begin{proof}
Recall that
$
d_Q=-Q\lambda_Q^*$, where $
\lambda_Q^*\in\mathop{\arg\min}_{\lambda\in\Delta^M}\|Q\lambda\|^2 .
$
Although \(\lambda_Q^*\) may not be unique, the vector \(Q\lambda_Q^*\) is
unique, since it is the minimum-norm point in the compact convex set
\(\operatorname{conv}\{q_{Z,1},\cdots,q_{Z,M}\}\). Therefore, \(d_Q\) is well-defined.
Let \(Q,Q'\in\mathbb{R}^{p\times M}\), and let
\begin{equation}
\lambda_Q^*\in\mathop{\arg\min}_{\lambda\in\Delta^M}\|Q\lambda\|^2,
\qquad
\lambda_{Q'}^*\in\mathop{\arg\min}_{\lambda\in\Delta^M}\|Q'\lambda\|^2.
\end{equation}
Using the optimality of \(Q'\lambda_{Q'}^*\) as the projection of the origin
onto \(\operatorname{conv}\{q'_{Z,1},\cdots,q'_{Z,M}\}\), we have
\begin{equation}
\left\langle
Q'\lambda_{Q'}^*,
Q'\lambda_{Q'}^*-Q'\lambda_Q^*
\right\rangle
\le 0 .    \label{eq:proj_ineq}
\end{equation}
Expanding the quadratic term and using~\eqref{eq:proj_ineq}, we obtain
\begin{align}
\|d_Q-d_{Q'}\|^2=\|Q\lambda_Q^*-Q'\lambda_{Q'}^*\|^2
&=
\|Q\lambda_Q^*\|^2-\|Q'\lambda_{Q'}^*\|^2
+2\left\langle
Q'\lambda_{Q'}^*,
Q'\lambda_{Q'}^*-Q\lambda_Q^*
\right\rangle \notag\\
&=
\|Q\lambda_Q^*\|^2-\|Q'\lambda_{Q'}^*\|^2
+2\left\langle
Q'\lambda_{Q'}^*,
Q'\lambda_{Q'}^*-Q'\lambda_Q^*
\right\rangle \notag\\
&\quad
+2\left\langle
Q'\lambda_{Q'}^*,
Q'\lambda_Q^*-Q\lambda_Q^*
\right\rangle \notag\\
&\le
\min_{\lambda\in\Delta^M}\|Q\lambda\|^2
-\min_{\lambda\in\Delta^M}\|Q'\lambda\|^2
+2\|Q'\lambda_{Q'}^*\|\,\|(Q'-Q)\lambda_Q^*\| .\label{eq:holder_after_projection}
\end{align}
It remains to bound the difference of optimal values. We do so by comparing the two quadratic objectives uniformly over \(\Delta^M\):
\begin{align}
\min_{\lambda\in\Delta^M}\|Q\lambda\|^2
-\min_{\lambda\in\Delta^M}\|Q'\lambda\|^2
&\le
\max_{\lambda\in\Delta^M}
\left|
\|Q\lambda\|^2-\|Q'\lambda\|^2
\right| \notag\\
&\le
\sup_{\lambda\in\Delta^M}\|(Q-Q')\lambda\|
\left(
\sup_{\lambda\in\Delta^M}\|Q\lambda\|
+
\sup_{\lambda\in\Delta^M}\|Q'\lambda\|
\right).\label{eq:holder_value_diff}
\end{align}
Combining~\eqref{eq:holder_after_projection} and
\eqref{eq:holder_value_diff}, and using
\(\|Q'\lambda_{Q'}^*\|\le
\sup_{\lambda\in\Delta^M}\|Q'\lambda\|\), gives
\begin{equation}
\|d_Q-d_{Q'}\|^2
\le
4\max\left\{
\sup_{\lambda\in\Delta^M}\|Q\lambda\|,
\sup_{\lambda\in\Delta^M}\|Q'\lambda\|
\right\}
\sup_{\lambda\in\Delta^M}\|(Q-Q')\lambda\|.
\label{eq:holder_general_bound}
\end{equation}
This is the \(1/2\)-H\"older bound established in~\cite{chen2024three}. On the bounded set considered in Lemma~\ref{lem:holder_sharp}, \(Q\) and \(Q'\) satisfy \(\|Q\|\le B\) and \(\|Q'\|\le B\). Hence,
\begin{equation}
\sup_{\lambda\in\Delta^M}\|Q\lambda\|\le B,
\qquad
\sup_{\lambda\in\Delta^M}\|Q'\lambda\|\le B,
\qquad
\sup_{\lambda\in\Delta^M}\|(Q-Q')\lambda\|\le \|Q-Q'\|.    
\end{equation}
Therefore, \eqref{eq:holder_general_bound} implies
\begin{equation}
\|d_Q-d_{Q'}\|^2
\le
4B\|Q-Q'\|,   \quad  \|d_Q-d_{Q'}\|
\le
2\sqrt{B}\,\|Q-Q'\|^{1/2}.
\end{equation}
Hence \(Q\mapsto d_Q\) is generally \(1/2\)-H\"older continuous on any bounded
set of gradient matrices. It remains to show that the exponent \(1/2\) is sharp. 
We prove the claim by an explicit counterexample. For \(t\in(0,\pi/2)\), define
\begin{equation}
Q_y(t)\coloneqq [q_1^y(t),q_2^y(t)],
\qquad
Q_z(t)\coloneqq [q_1^z(t),q_2^z(t)],
\end{equation}
where
$
q_1^y(t)=(\cos^2 t,\sin t\cos t)$, $ q_2^y(t)=(1,0)$
and
$
q_1^z(t)=(1,\cos t\sin t)$, $ q_2^z(t)=(1,0).
$
We first identify the corresponding CA directions. Observe that
\begin{equation}
q_1^y(t)-q_2^y(t)=(-\sin^2 t,\sin t\cos t).
\end{equation}
A direct computation gives
\begin{align}
\bigl(q_1^y(t)-q_2^y(t)\bigr)^\top q_1^y(t)
&=
(-\sin^2 t,\sin t\cos t)^\top (\cos^2 t,\sin t\cos t) \notag\\
&=
-\sin^2 t\cos^2 t+\sin^2 t\cos^2 t \notag\\
&=0.
\end{align}
Hence \(q_1^y(t)-q_2^y(t)\) is orthogonal to \(q_1^y(t)\). Consequently, \(q_1^y(t)\) is the projection of the origin onto \(\operatorname{conv}\{q_1^y(t),q_2^y(t)\}\), and the CA direction associated with \(Q_y(t)\) has the following explicit form.
\begin{equation}
d_{Q_y(t)}=-q_1^y(t)=-(\cos^2 t,\sin t\cos t).
\end{equation}
Similarly, $q_1^z(t)-q_2^z(t)=(0,\sin t\cos t)$ is orthogonal to $q_2^z(t)=(1,0)$, and hence $q_2^z(t)$ is the projection of the origin onto $\operatorname{conv}\{q_1^z(t),q_2^z(t)\}$. Thus
\begin{equation}
d_{Q_z(t)}=-q_2^z(t)=-(1,0).
\end{equation}
Therefore, the two CA directions differ by
\begin{align}
\|d_{Q_y(t)}-d_{Q_z(t)}\|
&=\|q_1^y(t)-q_2^z(t)\|
=\|(\sin^2t,-\sin t\cos t)\| 
=\sin t .
\label{eq:app_holder_d_diff}
\end{align}
On the other hand, the two matrices differ only in the first coordinate of their first column, since
\begin{align}
Q_y(t)-Q_z(t)
=\begin{bmatrix}
-\sin^2t & 0\\
0 & 0
\end{bmatrix},
\qquad
\|Q_y(t)-Q_z(t)\|=\sin^2t .
\label{eq:app_holder_Q_diff}
\end{align}
Combining \eqref{eq:app_holder_d_diff} and \eqref{eq:app_holder_Q_diff}, for any $\eta>\tfrac12$ we have
\begin{align}
\frac{\|d_{Q_y(t)}-d_{Q_z(t)}\|}{\|Q_y(t)-Q_z(t)\|^\eta}
=\frac{\sin t}{(\sin^2t)^\eta}
=(\sin t)^{1-2\eta}\to +\infty
\quad\text{as }t\to0^+ .
\end{align}
Thus no local H\"older bound with exponent $\eta>\tfrac12$ can hold uniformly for the mapping $Q\mapsto d_Q$. This proves the sharpness of the exponent $\tfrac12$ in the worst case.

We remark that this is consistent with Lemma~\ref{lem:lip_dq}. The Frobenius norms of \(Q_y(t)\) and \(Q_z(t)\) are uniformly bounded, for instance by \(B=2\). However, for \(U=(1,-1)^\top/\sqrt2\),
\begin{equation}
\mu_{\min}(U^\top Q_y(t)^\top Q_y(t)U)
=\frac12\|q_1^y(t)-q_2^y(t)\|^2
=\frac12\sin^2t\to0
\quad\text{as }t\to0^+,    
\end{equation}
and the same degeneracy occurs for \(Q_z(t)\). Thus no uniform \(\mu>0\) in Definition~\ref{def:pl_main} covers the family, and the loss of higher-order H\"older regularity occurs exactly where the nondegeneracy condition fails.
\end{proof}
\subsection{Proof of Theorem~\ref{thm:conv_main}}\label{app:conv_main}
\begin{proof}
Let \(\bar Q_t\coloneqq \nabla F_S(x_t)\),
\(
E_t\coloneqq Q_t-\bar Q_t
\),
\(
g_t\coloneqq Q_t\lambda_t=-d_t
\),
and define
\(
r(Q)\coloneqq \min_{\lambda\in\Delta^M}\|Q\lambda\|
\).
Thus, \(R_S(x_t)=r(\bar Q_t)^2\). The proof uses two geometric ingredients: the common-descent property of an MGDA direction and the Lipschitz continuity of the scalar map \(r\).

We first analyze the two branches of Algorithm~\ref{alg:alg} separately.\\
\textbf{Case 1 (regular branch).}
Suppose \(\mu_{\min}(U^\top Q_t^\top Q_tU)\ge\mu_t\). Then
\(\lambda_t=\lambda_{Q_t}^*\), and
\(g_t=Q_t\lambda_{Q_t}^*\) is the minimum-norm point in the convex hull of
the columns of \(Q_t\). Projection optimality gives
\(q_{Z_t,m}(x_t)^\top g_t\ge \|g_t\|^2\) for every \(m\in[M]\).
Taking a convex combination of these inequalities with \(\bar\lambda\in\Delta^M\) yields
\(\langle Q_t\bar\lambda,g_t\rangle\ge\|g_t\|^2\). Moreover, by the
definition of \(\lambda_{Q_t}^*\),
\(r(Q_t)=\|Q_t\lambda_{Q_t}^*\|=\|g_t\|\).\\
\textbf{Case 2 (non-regular branch).}
Suppose \(\mu_{\min}(U^\top Q_t^\top Q_tU)<\mu_t\). Algorithm~\ref{alg:alg} then
sets \(\lambda_t=\bar\lambda\), so \(g_t=Q_t\bar\lambda\). Consequently,
\(\langle Q_t\bar\lambda,g_t\rangle=\|g_t\|^2\). In addition, evaluating
the minimum in the definition of \(r(Q_t)\) at the feasible comparator
\(\bar\lambda\) gives
\(r(Q_t)\le\|Q_t\bar\lambda\|=\|g_t\|\).\\
\textbf{Unified argument.}
Thus, both branches satisfy
\begin{align}
\langle Q_t\bar\lambda,g_t\rangle&\ge \|g_t\|^2,
\label{eq:app_more_common_descent}\\
r(Q_t)&\le \|g_t\|.
\label{eq:app_more_norm_domination}
\end{align}
The first relation is exact in the non-regular branch. The second is exact in the regular branch.
Consider the fixed LS 
\(\Phi(x)\coloneqq\bar\lambda^\top F_S(x)\). By
Assumption~\ref{ass:smooth_main}, \(\Phi\) is
\(\ell_{f,1}\)-smooth. Since
\(\nabla\Phi(x_t)=\bar Q_t\bar\lambda=Q_t\bar\lambda-E_t\bar\lambda\),
\eqref{eq:app_more_common_descent} and Young's inequality yield
\begin{align}
\langle \nabla\Phi(x_t),g_t\rangle
&=
\langle Q_t\bar\lambda,g_t\rangle
-\langle E_t\bar\lambda,g_t\rangle \notag\\
&\ge
\frac12\|g_t\|^2-\frac12\|E_t\bar\lambda\|^2.
\label{eq:app_more_true_descent}
\end{align}
Using \(x_{t+1}=x_t-\alpha g_t\), smoothness, and
\(\alpha\ell_{f,1}\le 1/2\), we obtain
\begin{align}
\Phi(x_{t+1})
&\le
\Phi(x_t)-\alpha\langle\nabla\Phi(x_t),g_t\rangle
+\frac{\ell_{f,1}\alpha^2}{2}\|g_t\|^2 \notag\\
&\le
\Phi(x_t)-\frac{\alpha}{4}\|g_t\|^2
+\frac{\alpha}{2}\|E_t\bar\lambda\|^2.
\label{eq:app_more_scalar_descent}
\end{align}
Let \(\lambda_{Q'}\in\arg\min_{\lambda\in\Delta^M}\|Q'\lambda\|\). Since \(\lambda_{Q'}\in\Delta^M\) satisfies \(\|\lambda_{Q'}\|_2\le 1\),
\begin{align}
r(Q)-r(Q')
&\le
\|Q\lambda_{Q'}\|-\|Q'\lambda_{Q'}\| \notag\\
&\le
\|(Q-Q')\lambda_{Q'}\|
\le
\|Q-Q'\|_F .
\label{eq:app_ps_norm_lipschitz_one_side}
\end{align}
Thus, the scalar map \(r\) is Lipschitz in Frobenius norm:
\begin{equation}
|r(Q)-r(Q')|\le \|Q-Q'\|_F.
\label{eq:app_ps_norm_lipschitz}
\end{equation}
Applying~\eqref{eq:app_ps_norm_lipschitz} to \(Q_t\) and \(\bar Q_t\), and then using~\eqref{eq:app_more_norm_domination}, gives
\begin{align}
R_S(x_t)
=r(\bar Q_t)^2
&\le
\bigl(r(Q_t)+\|E_t\|_F\bigr)^2  
 \le
2\|g_t\|^2+2\|E_t\|_F^2.
\label{eq:app_more_ps_transfer}
\end{align}
Consequently, the following absorption inequality holds:
\begin{equation}
-\frac{\alpha}{4}\|g_t\|^2
\le
-\frac{\alpha}{8}R_S(x_t)
+\frac{\alpha}{4}\|E_t\|_F^2.
\label{eq:app_more_ps_absorption}
\end{equation}
Combining~\eqref{eq:app_more_scalar_descent} and~\eqref{eq:app_more_ps_absorption} and using
\(\|E_t\bar\lambda\|\le \|E_t\|_F\|\bar\lambda\|_2\le\|E_t\|_F\), yields the
pathwise recursion
\begin{equation}
\Phi(x_{t+1})
\le
\Phi(x_t)-\frac{\alpha}{8}R_S(x_t)
+\frac{3\alpha}{4}\|E_t\|_F^2.
\label{eq:app_more_pathwise_recursion}
\end{equation}
Conditioning on \(\mathcal F_t\), Assumption~\ref{ass:unbiased_main} implies
\begin{equation}
\mathbb E\!\left[\|E_t\|_F^2\mid\mathcal F_t\right]
=
\sum_{m=1}^M
\mathbb E\!\left[
\|q_{Z_t,m}(x_t)-\nabla f_{S,m}(x_t)\|^2
\mid\mathcal F_t
\right]
\le
\frac{M\sigma^2}{|Z_t|}.
\label{eq:conditional_second_moment_noise}
\end{equation}
Taking conditional expectations in~\eqref{eq:app_more_pathwise_recursion}, rearranging, and taking total expectations gives
\begin{equation}
\alpha\,\mathbb E[R_S(x_t)]
\le
8\,\mathbb E[\Phi(x_t)-\Phi(x_{t+1})]
+\frac{6M\sigma^2\alpha}{|Z_t|}.
\label{eq:app_more_expected_recursion}
\end{equation}
Summing~\eqref{eq:app_more_expected_recursion} over
\(t=0,\ldots,T-1\) and using
\(\Delta_0=\Phi(x_0)-\inf_x\Phi(x)\) gives
\begin{equation}
\frac1T\sum_{t=0}^{T-1}\mathbb E[R_S(x_t)]
\le
\frac{8\Delta_0}{\alpha T}
+\frac{6M\sigma^2}{T}
\sum_{t=0}^{T-1}\frac{1}{|Z_t|}.
\label{eq:app_more_average_rate}
\end{equation}
Finally, the best-iterate bound satisfies
\begin{equation}
\mathbb E\!\left[\min_{0\le t\le T-1}R_S(x_t)\right]
\le
\frac1T\sum_{t=0}^{T-1}\mathbb E[R_S(x_t)],
\label{eq:app_more_best_iterate}
\end{equation}
so~\eqref{eq:app_more_average_rate} proves~\eqref{eq:thm_general_rate}. If \(|Z_t|\ge b(t+1)\) for some \(b>0\), then
\begin{equation}
\sum_{t=0}^{T-1}\frac{1}{|Z_t|}
\le
\frac{1}{b}\sum_{t=0}^{T-1}\frac{1}{t+1}
\le
\frac{1+\log T}{b}.
\label{eq:app_more_harmonic_sum}
\end{equation}
Substituting~\eqref{eq:app_more_harmonic_sum} into~\eqref{eq:app_more_average_rate} proves~\eqref{eq:thm_final_rate_tilde} for any \(\alpha=\Theta(1)\) with \(\alpha\le1/(2\ell_{f,1})\). No step in the argument depends on \(\mu_t\).
\end{proof}

\subsection{Proof of Theorem~\ref{thm:ca}}\label{app:ca}
\begin{proof}
Since \(M=2\), write
\(\nabla F_S(x_t)=[\nabla f_{S,1}(x_t),\nabla f_{S,2}(x_t)]\) and
\(Q_t=[q_{Z_t,1}(x_t),q_{Z_t,2}(x_t)]\). For brevity in this proof, write \(q_{Z_t,m}=q_{Z_t,m}(x_t)\) for \(m=1,2\). Let
\(d(x_t)=-\nabla F_S(x_t)\lambda^*(x_t)\), where
\(\lambda^*(x_t)\in\mathop{\arg\min}_{\lambda\in{\Delta^2}}
\|\nabla F_S(x_t)\lambda\|^2\), and let
\(d_t=-Q_t\lambda_t\).
Define the regular branch event
$
\mathcal R_t
\coloneqq
\left\{
\mu_{\min}(U^\top Q_t^\top Q_tU)\ge \mu_t
\right\}$.
Then
$
\mathcal R_t^c
=
\left\{
\mu_{\min}(U^\top Q_t^\top Q_tU)<\mu_t
\right\}
$. 
By the definition of \(\mathcal E_{\mathrm{ca}}(t)\) and Jensen's inequality,
\begin{align}
\mathcal E_{\mathrm{ca}}(t)
=\left\|\mathbb E[d_t-d(x_t)]\right\|^2=
\left\|
\mathbb E\!\left[
\mathbb E[d_t-d(x_t)\mid\mathcal F_t]
\right]
\right\|^2
\le
\mathbb E\!\left[
\left\|
\mathbb E[d_t-d(x_t)\mid\mathcal F_t]
\right\|^2
\right]
\le
\E\!\left[\|d_t-d(x_t)\|^2\right] .
\label{eq:app_ca_bound_1}
\end{align}
We split the right-hand side according to the two branches of
Algorithm~\ref{alg:alg}:
\begin{align}
\mathbb E\!\left[\|d_t-d(x_t)\|^2\right]
&=
\mathbb E\!\left[
\mathbf 1_{\mathcal R_t}\|d_t-d(x_t)\|^2
\right]
+
\mathbb E\!\left[
\mathbf 1_{\mathcal R_t^c}\|d_t-d(x_t)\|^2
\right].
\label{eq:ca_branch_split}
\end{align}
The two terms require different arguments. Global \(1/2\)-H\"older
continuity controls the regular branch. In the non-regular branch, the
two-objective curvature identity controls the fallback error directly.\\
\textbf{Case 1 (regular branch).}
Suppose \(\mu_{\min}(U^\top Q_t^\top Q_tU)\ge\mu_t\). On this branch,
\(d_t=d_{Q_t}\). Assumption~\ref{ass:gradbound_main}
and~\eqref{eq:uniform_sample_gradient_consequences} give
\(\|Q_t\|_F\le\ell_F\) and
\(\|\nabla F_S(x_t)\|_F\le\ell_F\). The squared form of
Lemma~\ref{lem:holder_sharp} therefore gives
\begin{equation}
\mathbf 1_{\mathcal R_t}\|d_t-d(x_t)\|^2
\le
4\ell_F\|Q_t-\nabla F_S(x_t)\|_F.
\label{eq:ca_regular_holder_bound}
\end{equation}
Since \(M=2\), the conditional Cauchy--Schwarz inequality and
Assumption~\ref{ass:unbiased_main} imply
\begin{align}
\E\!\left[
\|Q_t-\nabla F_S(x_t)\|_F
\mid\mathcal F_t
\right]
&\le
\left(
\E\!\left[
\|Q_t-\nabla F_S(x_t)\|_F^2
\mid\mathcal F_t
\right]
\right)^{1/2}\notag\\
&\le
\frac{\sqrt{2}\,\sigma}{\sqrt{|Z_t|}}.
\label{eq:ca_first_moment_noise}
\end{align}
Taking expectations in~\eqref{eq:ca_regular_holder_bound} and
using~\eqref{eq:ca_first_moment_noise} gives
\begin{equation}
\E\!\left[
\mathbf 1_{\mathcal R_t}\|d_t-d(x_t)\|^2
\right]
\le
\frac{4\sqrt{2}\,\ell_F\sigma}{\sqrt{|Z_t|}},
\label{eq:app_ca_rate_1}
\end{equation}
which is independent of \(\mu_t\).\\
\textbf{Case 2 (non-regular branch).}
If \(\mu_{\min}(U^\top Q_t^\top Q_tU)< \mu_t\),
every \(\lambda\in{\Delta^2}\) can be written as \(\lambda=(a,1-a)\) for some
\(a\in[0,1]\). For $\lambda_t$ and $\lambda^*(x_t)$, write
\begin{equation}
\lambda_t=(\hat a_t,1-\hat a_t),
\qquad
\lambda^*(x_t)=(a_t^*,1-a_t^*),
\end{equation}
with \(a_t^*,\hat a_t\in[0,1]\). Then
\begin{equation}
d_t
=-(\hat a_t q_{Z_t,1}+(1-\hat a_t)q_{Z_t,2}),
\qquad
d(x_t)
=-(a_t^* \nabla f_{S,1}(x_t)+(1-a_t^*)\nabla f_{S,2}(x_t)).
\end{equation}
Adding and subtracting \(a_t^*q_{Z_t,1}+(1-a_t^*)q_{Z_t,2}\), we obtain
\begin{align}
d_t-d(x_t)
&=
(a_t^*-\hat a_t)(q_{Z_t,1}-q_{Z_t,2})
+a_t^*(\nabla f_{S,1}(x_t)-q_{Z_t,1})
+(1-a_t^*)(\nabla f_{S,2}(x_t)-q_{Z_t,2}).
\end{align}
Thus, using \(a_t^*,\hat a_t\in[0,1]\),
\begin{align}
\|d_t-d(x_t)\|
&\le
|a_t^*-\hat a_t|\,\|q_{Z_t,1}-q_{Z_t,2}\|
+a_t^*\|\nabla f_{S,1}(x_t)-q_{Z_t,1}\| 
+(1-a_t^*)\|\nabla f_{S,2}(x_t)-q_{Z_t,2}\| \notag\\
&\le
\|q_{Z_t,1}-q_{Z_t,2}\|
+\|\nabla f_{S,1}(x_t)-q_{Z_t,1}\|
+\|\nabla f_{S,2}(x_t)-q_{Z_t,2}\|.
\end{align}
Squaring both sides and using
\(\|u+v+w\|^2\le 3\|u\|^2+3\|v\|^2+3\|w\|^2\), we get
\begin{equation}
\|d_t-d(x_t)\|^2
\le
3\|q_{Z_t,1}-\nabla f_{S,1}(x_t)\|^2
+3\|q_{Z_t,2}-\nabla f_{S,2}(x_t)\|^2
+3\|q_{Z_t,1}-q_{Z_t,2}\|^2.
\label{eq:app_twoobj_sq}
\end{equation}
For \(M=2\), the tangent space of the simplex is
\(\operatorname{span}\{(1,-1)^\top\}\). The reduced curvature matrix \(U^\top Q_t^\top Q_tU\) is a scalar, 
and
\begin{equation}
\mu_{\min}(U^\top Q_t^\top Q_tU)=
\frac12 (q_{Z_t,1}-q_{Z_t,2})^\top(q_{Z_t,1}-q_{Z_t,2})=
\frac12\|q_{Z_t,1}-q_{Z_t,2}\|^2.
\label{eq:ca_reduced_curvature_m2}
\end{equation}
Equivalently, for \(v=(1,-1)^\top\), we have
$
Q_tv=q_{Z_t,1}-q_{Z_t,2},
$
and \(U=v/\sqrt 2\). 
Therefore, under the case condition
$\mu_{\min}(U^\top Q_t^\top Q_tU)< \mu_t$,
we have
$
\|q_{Z_t,1}-q_{Z_t,2}\|^2<2\mu_t.  
$
Multiplying \eqref{eq:app_twoobj_sq} by \(\mathbf 1_{\mathcal R_t^c}\), bounding the last term via \(\mathbf 1_{\mathcal R_t^c}\|q_{Z_t,1}-q_{Z_t,2}\|^2<2\mu_t\), dropping the indicator on the two variance terms, and taking expectations yields
\begin{align}
\mathbb E\!\left[
\mathbf 1_{\mathcal R_t^c}\|d_t-d(x_t)\|^2
\right]
\le
3\E\!\left[\|q_{Z_t,1}-\nabla f_{S,1}(x_t)\|^2\right] 
+3\E\!\left[\|q_{Z_t,2}-\nabla f_{S,2}(x_t)\|^2\right]
+6\mu_t.
\label{eq:app_case2_prevariance}
\end{align}
By Assumption~\ref{ass:unbiased_main}, for \(m=1,2\),
$
\E\!\left[\|q_{Z_t,m}-\nabla f_{S,m}(x_t)\|^2\right]
\le
\frac{\sigma^2}{|Z_t|}.
$
Combining this fact and \eqref{eq:app_case2_prevariance}, we obtain
\begin{equation}
\mathbb E\!\left[
\mathbf 1_{\mathcal R_t^c}\|d_t-d(x_t)\|^2
\right]
\le
\frac{6\sigma^2}{|Z_t|}+6\mu_t.
\label{eq:app_ca_rate_2}
\end{equation}
\textbf{Unified bound.}
Combining~\eqref{eq:app_ca_rate_1} and~\eqref{eq:app_ca_rate_2}
with~\eqref{eq:app_ca_bound_1} gives
\begin{equation}
\mathcal E_{\mathrm{ca}}(t)
\le
\frac{4\sqrt{2}\,\ell_F\sigma}{\sqrt{|Z_t|}}
+
\frac{6\sigma^2}{|Z_t|}
+
6\mu_t
=
\mathcal O\!\left(
\frac{1}{\sqrt{|Z_t|}}
+
\frac{1}{|Z_t|}
+
\mu_t
\right).
\end{equation}
For \(|Z_t|=\Theta(t+1)\) and
\(\mu_t=\Theta((t+1)^{-1/2})\), this becomes
\begin{equation}
\mathcal E_{\mathrm{ca}}(t)
=
\mathcal O((t+1)^{-1/2}).
\end{equation}
\end{proof}

\vfill

\end{document}